\RequirePackage[loading]{tracefnt}
\documentclass[letterpaper, 10 pt, conference]{ieeeconf}  

\IEEEoverridecommandlockouts                              

\usepackage{cite}
\usepackage{amsmath,amssymb,amsfonts}
\usepackage{graphicx}
\usepackage{textcomp}
\def\BibTeX{{\rm B\kern-.05em{\sc i\kern-.025em b}\kern-.08em
    T\kern-.1667em\lower.7ex\hbox{E}\kern-.125emX}}
\usepackage{bm}
\usepackage{hyperref}
\usepackage{amsmath}
\usepackage{cleveref}
\usepackage{amssymb}
\usepackage{textcomp}
\usepackage{algorithm}
\usepackage[noend]{algpseudocode}
\usepackage{stackengine}
\usepackage{hyperref}
\usepackage{graphicx}
\usepackage{framed}
\usepackage{wrapfig}
\usepackage{array}
\usepackage{subcaption}
\usepackage{booktabs}
\usepackage{comment}
\usepackage{siunitx}
\usepackage[font=small,labelfont=bf]{caption}
\newcolumntype{L}[1]{>{\raggedright\arraybackslash}p{#1}}
\algnewcommand{\algorithmicor}{\textbf{ or }}
\algnewcommand{\OR}{\algorithmicor}

\usepackage{tikz}
\usetikzlibrary{external}
\crefname{equation}{}{}
\Crefname{equation}{}{}
\crefname{figure}{Fig.}{Fig.}

\newcommand{\ie}{\textit{i}.\textit{e}., }

\newcommand{\R}{\mathbb{R}}

\definecolor{brickred}{rgb}{0.8, 0.25, 0.33}
\newcommand{\rebuttal}[1]{{\color{black} #1}}

\newtheorem{theorem}{Theorem}
\newtheorem{remark}{Remark}

\newtheorem{definition}{Definition}
\newtheorem{problem}{Problem}

\Crefname{asm}{Assumption}{Assumption}

\newcommand{\revision}[1]{{\color{black} #1}}

\DeclareMathOperator*{\argmin}{arg\,min}

\title{\LARGE \bf
Safety Index Synthesis via Sum-of-Squares Programming
}

\author{Weiye Zhao$^1$, Tairan He$^1$, Tianhao Wei, Simin Liu and Changliu Liu
\thanks{1. Equal contribution.}
\thanks{W. Zhao, T. He, T. Wei, S. Liu and C. Liu are with the Robotics Institute, Carnegie
Mellon University, Pittsburgh, PA 15213 USA (e-mail: {\tt\small weiyezha, tairanh, twei2, siminliu, cliu6@andrew.cmu.edu}).}%
\thanks{This material is based upon work supported by the National Science Foundation under Grant No. 2144489.}
}

\begin{document}

\maketitle
\thispagestyle{empty}
\pagestyle{empty}

\begin{abstract}
Control systems often need to satisfy strict safety requirements. 
Safety index provides a handy way to evaluate the safety level of the system and derive the resulting safe control policies.
However, designing safety index functions under control limits is difficult and requires a great amount of expert knowledge.
This paper proposes a framework for synthesizing the safety index for general control systems using sum-of-squares programming. Our approach is to show that ensuring the non-emptiness of safe control on the safe set boundary is equivalent to a local manifold positiveness problem. 
We then prove that this problem is equivalent to sum-of-squares programming via the Positivstellensatz of algebraic geometry. 
We validate the proposed method on robot arms with different degrees of freedom and ground vehicles. The results show that the synthesized safety index guarantees safety and our method is effective even in high-dimensional robot systems.
\end{abstract}


\section{Introduction}
Energy-function-based algorithms~\cite{wei2019safe} have been widely studied as appealing tools for safe control. When properly designed, the energy function (also called the safety index or barrier function) maps dangerous states to high values and safe states to low values. Then, safety can be ensured by finding a control input that dissipates the energy. However, when control limits exist, this desired safe control may not be realizable and thus the safety guarantee may be broken. Therefore, it is important to account for control limits in energy function design.

Accounting for control limits in safety index design is challenging, because we have to ensure the feasibility of safe control at \textit{infinitely many} states. This feasibility constraint itself can be quite complex, as it is a function of the control limits, system dynamics, and safety requirements. There have been various approaches to safety index synthesis under control limits. Some works assume that the dynamical system has a special structure (i.e. kinematic bicycle, Euler-Lagrange), which enables hand-derivation of a safety index~\cite{zhao2021issa, cortez2020safe}. Another approach is to formulate synthesis as optimizing parameters within some safety index function form~\cite{clark2021verification,wei2022safe}. However, these methods do not scale well to high-dimensional systems. 

Sum-of-Squares Programming (SOSP) has been applied to safety index synthesis~\cite{clark2021verification,ames2019control}, as it is a powerful tool for dealing with  optimization problems with infinitely many constraints. The issue with existing SOSP-based synthesis techniques is that they try to enforce existence of safe control for \textit{all states in the state space}~\cite{ames2019control,clark2021verification}. This is otherwise known as enforcing a \textit{global positiveness} constraint~\cite{zhao2022provably}. 

In this work, we highlight that instead of satisfying \textit{global positiveness} constraints, it is sufficient for the safety index to satisfy \textit{local manifold positiveness} constraints, which enables the corresponding safety index design to be much less conservative. That means we can find solutions (a valid safety index) more often. Subsequently, we propose a general, scalable, efficient SOSP-based method to transform
safety index synthesis problem into a nonlinear programming problem that can be solved efficiently by off-the-shelf solvers. 


The experiments show that the synthesized safety index ensures the existence of safe control and that our method is time-efficient, even for high-dimensional robotic systems.

\textbf{Our contributions} We propose a novel SOSP-based method that can efficiently synthesize a safety index. The state-of-the-art SOSP-based methods are (i) limited to systems with polynomial or sinusoidal nonlinearities~\cite{ames2019control}; (ii) hard to scale to high-dimensional systems (e.g., iteratively solving $\mathcal{O}(2k)$ SOS programs at $k$-th iteration~\cite{clark2021verification}); (iii) limited to conservative solutions due to considering \textit{global positiveness} constraints. Our algorithm (i) is applicable to polynomial substitutable (Def.~\ref{def: subs}) control systems; (ii) scales well to high-dimensional systems (solving only one nonlinear programming) and (iii) gives nonconservative solutions via solving \textit{local manifold positiveness} constraints. 

\rebuttal{In the remainder of the paper, we first discuss related work about energy-function-based methods and safe control in \Cref{sec: related_work}. We then formulate the mathematical problem for safety index synthesis in \Cref{sec: formulation}. In \Cref{sec: method} and \Cref{sec:proof}, we first introduce the proposed optimization
algorithm based on SOSP, and provide theoretical results that the algorithm can obtain a feasible safety index design. Finally, we validate our proposed method in high-dimensional robot systems in \Cref{sec: exp}. Our code is available on Github.\footnote{\url{https://github.com/intelligent-control-lab/Safety-Index-Synthesis-via-Sum-of-Squares-Programming}}
} 

\section{Related Work}\label{sec: related_work}
In literature, many different energy functions~\cite{wei2019safe} are proposed to measure safety, including (i) potential function; 
(ii) safety index;
(iii) control barrier function. Representative methods include potential field method,
barrier function method~\cite{liu2022safe}, and safe set algorithm~\cite{liu2014control}. 

However, designing such energy functions is difficult and requires great human efforts to find appropriate parameters and function forms. 
As for synthesizing \textit{safety index}, 
\rebuttal{
\cite{liu2014control} discusses the general safety index design rule that guarantees forward invariance of safety. However, the safety guarantee comes from the assumption of unbounded control input.}
\cite{zhao2021issa} solves the problem based on worst-case analysis, but it is limited to simple 2D mobile robot dynamics. \cite{wei2022safe} leverages an evolutionary algorithm for safety index parameterization, but the computation time of grid sampling increases exponentially with the dimension of the state space.
\rebuttal{\cite{ma2022joint} proposes the joint synthesis of safety index and safe control policy, such that safe control can be generated for all time steps. However, the safety index is optimized via gradient descent, where only local optima convergence is guaranteed under strong assumptions.
}

As for synthesizing \textit{control barrier functions} (CBFs), safety index is closely related to CBFs as analyzed in~\cite{wei2019safe}. Classical hand-designed CBFs
are hard to scale to nonlinear or high-dimensional systems.
\rebuttal{\cite{liu2022safe} proposes synthesis of neural CBF that takes account of the control limits. The neural CBF is learnt via iterative optimization to eliminate the counterexamples of input saturation. However, such adversarial optimization is prone to local optima and no theoretical guarantee is provided. 
}
Some automated methods
~\cite{clark2021verification} 
leverage SOSP to synthesize CBFs. However, \cite{clark2021verification} is limited to systems with polynomial dynamics. 
Existing techniques for synthesizing barrier certificates using SOSP are inapplicable to design CBFs~\cite{clark2021control}. To this end, ~\cite{clark2021verification} develop a SOSP-based program via Positivstellensatz to verify CBFs, but it is also computationally challenging for high-dimensional systems.
To designs certificate functions in complex systems, another line of work~(\cite{liu2022safe, dawson2022safe}) proposes to learn approximate CBFs with neural networks, but these methods lack theoretical guarantees of safety in practice.
\rebuttal{CBVFs~\cite{choi2021robust} unifies CBFs and Hamilton-Jocobi reachability, handling bounded control and disturbances. However, this comes with a cost of bearing the curse of dimensionality.}


\section{Problem Formulation} \label{sec: formulation}

This section introduces system dynamics and safety specification, and formulates the safety index synthesis problem.

\paragraph*{Dynamics}
 Let $x \in \mathcal{X}\subset \mathbb{R}^{n_x}$ be the robot state, where $n_x$ is the dimension of the bounded state space $\mathcal{X}$, \rebuttal{and $x \in \mathcal{X}$ can  be represented by a set of inequalities $\{S(x)_i \geq 0\}, i = 1,2,\cdots,n_S$.} 
 
 Let $u \in \mathcal{U}\subset \mathbb{R}^{n_u}$ be the control input to the robot, where $n_u$ is the dimension of the control space $\mathcal{U}$, and $\mathcal{U}$ is bounded by a $n_u$-dimensional orthotope, i.e. 
\begin{align}
    \forall i = 1,2,\cdots,n_u, \, \mathcal{U}^{[i]} \in [u^{[i]}_{min}, u^{[i]}_{max}],
\end{align}
where $\mathcal{U}^{[i]}$ denotes the $i$-th dimensional control space, and $u^{[i]}_{min}, u^{[i]}_{max} \in \mathbb{R}$. In this paper, we consider a control-affine system, and the system dynamics are defined as: 
\begin{equation}\label{eq:dynamics fn}
\begin{split}
    &\dot{x} = f(x) + g(x) u, \\
\end{split}
\end{equation}
where $f: \R^{n_x} \to \R^{n_x}$, $g: \R^{n_x} \to \R^{n_x \times n_u}$ are locally Lipschitz continuous on $\R^{n_x}$. Here we highlight that $f(x), g(x)$ are not necessarily polynomials, and they satisfy the following definition:


\begin{definition}(Substitutable Function)
\label{def: subs}
A function is substitutable if its variables can be projected into the higher dimensional polynomial manifold, such that the function can be rewritten as polynomials.
\end{definition}

\begin{remark}
To better understand \Cref{def: subs}, we give an example of a substitutable function. Consider function
\begin{align*}
    f^*(\theta) := \sin\theta \cos\theta,
\end{align*}
by substituting $\sin\theta$ with $\mathfrak{a}$, and $\cos\theta$ with $\mathfrak{b}$, respectively. $f^*(\theta)$ can be rewritten as following polynomial
\begin{align*}
    f^*(\mathfrak{a},\mathfrak{b}) := \mathfrak{a}\mathfrak{b},
\end{align*}
where $\mathfrak{a}^2 + \mathfrak{b}^2 - 1 = 0$.
\end{remark}



\paragraph*{Safety Specification}
The safety specification requires the system state should be constrained in a closed subset in the state space, called the safe set $\mathcal{X}_S$. The safe set can be represented by the zero-sublevel set of a continuous and piecewise smooth function $\phi_0:\mathbb{R}^{n_x} \rightarrow \mathbb{R}$, i.e., $\mathcal{X}_S = \{x\mid \phi_0(x)\leq 0\}$. $\mathcal{X}_S$ and $\phi_0$ are directly specified by users. The design of $\phi_0$ is straightforward in most scenarios. For example, for collision avoidance, $\phi_0$ can be designed as the negative closest distance between the robot and environmental obstacles.

\paragraph*{Problem}
Since $\mathcal{X}_S$ 
may contain states that will inevitably go to the unsafe set no matter what the control input is, we need to assign high energy values to these inevitably-unsafe states. Then, safe control can be applied to 
ensure \textbf{forward invariance} in a subset of the safe set $\mathcal{X}_S$ by dissipating the energy.
\textit{Forward invariance} of a set means that the robot state will never leave the set if it starts from the set, i.e. when $\phi_0(x(t_0)) \leq 0$, then $\phi_0(x(t)) \leq 0,\ \forall t>t_0$.



In this paper, we adopt safe set algorithm (SSA) ~\cite{liu2014control}, which is an \textit{energy function-based method} for safe control.
SSA has introduced a rule-based approach to synthesize the energy function as a continuous, piece-wise smooth scalar function $\phi:\mathbb{R}^{n_x} \rightarrow \mathbb{R}$. And the energy function $\phi(x)$ is called a safety index. 
\rebuttal{The general form of the safety index is $\phi = \phi_0 + k_1\phi_0^{(1)} + \cdots + k_n\phi_0^{(n)}$} where 1) the roots of $1+k_1s+\ldots + k_n s^n=0$ are all negative real (to ensure zero-overshooting of the original safety constraints); 2) the relative degree from \rebuttal{n-order derivative} $\phi_0^{(n)}$ to $u$ is one (to avoid singularity). 
In our paper, we assume $\frac{\partial \phi}{\partial x}$ satisfies \Cref{def: subs}.

It is shown in \cite{liu2014control} that if the control input is unbounded ($\mathcal{U} =\mathbb{R}^{n_u}$), then there always exist a control \rebuttal{$u$} that satisfies the constraint \rebuttal{$\dot\phi(x,u) \leq 0 \text{ when } \phi= 0$}, \rebuttal{where $\dot\phi(x,u)$ denotes the derivative of $\phi(x)$ with respect to time under control $u$. For simplicity, we use $\dot \phi$ to represent $\dot \phi(x,u)$}. If the control input always satisfies that constraint, then the set $\{x\mid\phi(x)\leq 0\}\cap \{x\mid \phi_0(x)\leq 0\}$ is forward invariant.
In practice, when $\phi=0$, the safe control $u$ is computed through a quadratic projection of the nominal control $u^r$:
\begin{align}\label{eq: quadratic program for control}
    u =& \argmin_{u\in\mathcal{U}} \|u-u^r\|^2\text{ s.t. } \dot\phi \leq 0.
\end{align}

\begin{figure}[htbp]
    \centering
    \includegraphics[width=0.8\columnwidth]{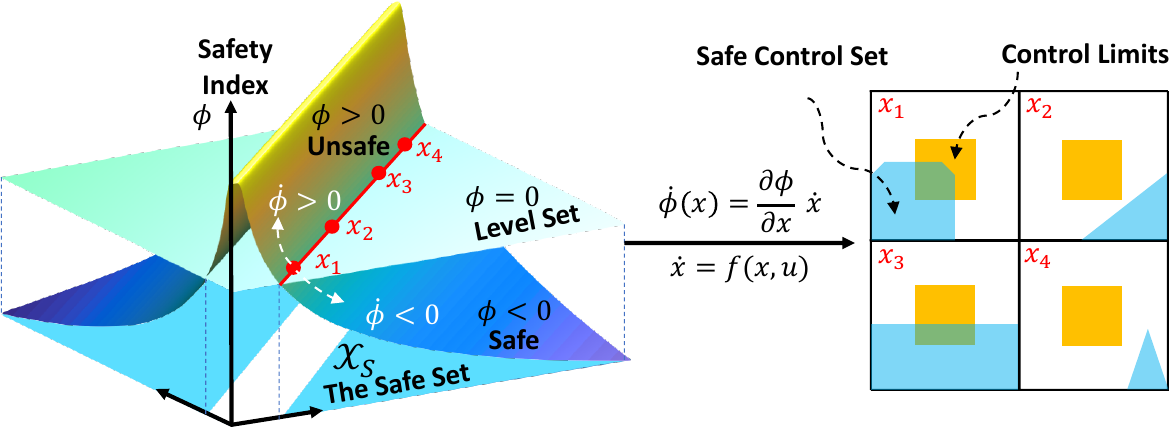}
    \caption{\rebuttal{The manifold of safety index and sets of safe control.}}
    \vspace{-15pt}    \label{fig:problem_illustration}
\end{figure}
The issue with \eqref{eq: quadratic program for control} is that this safe controller can saturate as shown in \Cref{fig:problem_illustration} where the sets of safe control may have no intersection with the limited control space. Specifically, saturation occurs when  there does not exist a control input that satisfies the constraint $\dot\phi \leq 0 \text{ when } \phi= 0$,
causing the loss of safety guarantees. This can happen because we have not yet accounted for control limits in our design of $\phi$. Hence, the core problem of this paper is to choose a proper parameterization of safety index that guarantees the existence of safe control (within control limits) to ensure $\dot\phi \leq 0 \text{ for all states where } \phi= 0$. We call this process as \textit{Safety Index Synthesis with Control Limits}. Mathematically, \textit{Safety Index Synthesis} solves the following problem:
\begin{problem}[\textit{Safety Index Synthesis}]
\label{def:problem 1}
Construct safety index as $\phi_\Theta = \phi_0 + k_1\dot\phi_0 + \cdots + k_n\phi_0^{(n)}$, with tunable parameters set $\Theta = \{k_1, k_2, \cdots, k_n\}$ and $\forall i = 1,2,\cdots,n, k_i \in \mathbb{R}^+$, such that
\begin{align}
\label{eq: problem_ori}
    \forall x \in \mathcal{X} \text{ s.t. } \phi_\Theta(x) = 0, \min_{u \in \mathcal{U}} \dot{\phi}(x, u) < 0.
\end{align}
\end{problem}
However, solving Problem \ref{def:problem 1} is not a trivial task. Note that \eqref{eq: problem_ori} contains infinitely many constraints since for every state $x, \text{s.t. } \phi_\Theta(x) = 0$, poses an inequality constraint, hence cannot be directly solved by off-the-shelf nonlinear programming solvers. To deal with this issue, we leverage tools from SOSP (to be reviewed in the following section) to repose the problem as a nonlinear program.

\section{\rebuttal{Background} of Sum of Square Programming}\label{sec:sos}

Optimization problems with \textit{global positiveness constraints}~\cite{zhao2022provably} in the form of finding a function $F(x), \text{ s.t. } \forall x\in\mathbb{R}^{n_x}, G\circ F(x)>0$, where $G$ is a function,
have been widely studied and can be solved by SOSTOOLs~\cite{zhao2022provably}.
However, our problem requires $\min_{u \in \mathcal{U}} \dot{\phi}(x, u)$ to be nonpositive on a limited-size manifold, i.e. $x\in \mathcal{X}, \phi(x) = 0$. Hence, Problem \ref{def:problem 1} cannot be directly solved by SOSTOOLS.
Instead, we will reconsider the theory behind SOSP to solve Problem \ref{def:problem 1}.

To ensure \textit{global positiveness} of a condition, the easiest way is to show that there does not exist any \revision{solution such that} the condition is violated (refute set). 
Constructing the refute set and showing it is empty is the core idea behind SOSP~\cite{parrilo2003semidefinite} to ensure \textit{global positiveness}. To show that the refute set is empty, we need to invoke the equivalence conditions in \textit{Positivstellensatz}~\cite{parrilo2003semidefinite}. 
Before introducing \textit{Positivstellensatz}, we first review a key concept: ring-theoretic \textit{cone}~\cite{parrilo2003semidefinite}.
  
  
  \rebuttal{
  \begin{definition}[Ring-theoretic cone]\label{prop}Denote $\mathbb{R}[x_1,\ldots,x_n]$ a set of polynomials with $[x_1,\ldots,x_n]$ as variables.
  For a 
  set $S = {\gamma_1,\ldots,\gamma_s} \subseteq \mathbb{R}[x_1,\ldots,x_n]$, the associated ring-theoretic \textit{cone} can be expressed as: 
  \begin{equation} \nonumber
      \Gamma = \{ p_0 + p_1\gamma_1 + \ldots + p_s\gamma_s + p_{12}\gamma_1\gamma_2 +\ldots + p_{12\ldots s}\gamma_1\ldots\gamma_s\}~,
  \end{equation}
  where $p_0, \ldots,p_{12\ldots s}$
are the polynomials that are SOS.
  \end{definition}
  }
  
  Based on the ring-theoretic \textit{cone},  \textit{Positivstellensatz} condition is specified in the following theorem.

  \begin{theorem}[\textit{Positivstellensatz}] \label{theo:posi}
  Let $(\gamma_j)_{j=1,\ldots,s}$, $(\psi_k)_{k=1,\ldots,t}$, $(\zeta_l)_{l=1,\ldots,r}$ be finite families of polynomials in $\mathbb{R}[x_1,\ldots,x_n]$. Let $\Gamma$ be the ring-theoretic cone generated by $(\gamma_j)_{j=1,\ldots,s}$, $\Xi$ the multiplicative monoid~\cite{parrilo2003semidefinite} generated by $(\psi_k)_{k=1,\ldots.,t}$, and $\Phi$ the ideal~\cite{parrilo2003semidefinite} generated by $(\zeta_l)_{l=1,\ldots,r}$. Then, the following properties are equivalent:
  \begin{enumerate}
      \item The following set is empty
          \begin{align}\label{eq: them1 empty set}
          \begin{cases}
            x \in \mathbb{R}^n \Biggr|\begin{array}{ll}
                                    \gamma_j(x)\geq 0,j = 1,..,s\\
                                    \psi_k(x)\neq 0, k = 1,\ldots,t\\
                                    \zeta_l(x)=0, l = 1,\ldots,r
                                    \end{array}
          \end{cases}.
          \end{align}
      \item There exist  $\gamma \in \Gamma, \psi\in \Xi, \zeta \in \Phi$ such that 
      \begin{align}
          \gamma+\psi^2+\zeta = 0,
          \label{eq:pos}
      \end{align}
  \end{enumerate}
  where $\psi = 1$ if 
  $t = 0$.
  \end{theorem}

  \revision{We refer the reader to \cite{parrilo2003semidefinite} for the proofs of \Cref{theo:posi}.} In summary, \textit{Positivstellensatz} shows that the refute set being empty is equivalent to a feasibility problem in \eqref{eq:pos}. 
  Therefore, we can follow the same procedure
  to construct a refute set for
Problem \ref{def:problem 1}
  and then use \textit{Positivstellensatz} to turn the problem into a feasibility problem similar to \eqref{eq:pos}, which can then be formed into an ordinary nonlinear program.

\section{Method} \label{sec: method}
In this section, we will introduce four steps to solve Problem \ref{def:problem 1} leveraging \Cref{theo:posi}. 

\subsection{The Local Manifold Positiveness Problem}
\label{sec:lmp}
Since \Cref{theo:posi} applies to a set of equalities, inequations, and inequalities, we need to unroll $\min_{u \in \mathcal{U}} \dot{\phi}(x, u)$ and get rid of 
$\min$
operator. Firstly, the equivalent form of $\min_{u \in \mathcal{U}} \dot{\phi}(x, u)$ is summarized as following:
\begin{align}
\label{eq:infimum dot phi ori}
    \min_{u \in \mathcal{U}} \dot{\phi}(x, u) &= \min_{u \in \mathcal{U}} \underbrace{\frac{\partial \phi}{\partial x}f(x)}_{{L_f\phi}} + \underbrace{\frac{\partial \phi}{\partial x}g(x)}_{{L_g\phi}}u \\ \nonumber 
    &= \min_{u \in \mathcal{U}} {L_f\phi} + \sum_{i = 1}^{n_u} {L_g\phi}^{[i]}u^{[i]},
\end{align}
where ${L_g\phi}^{[i]}$ and $u^{[i]}$ denote the $i$-th dimension of ${L_g\phi}$ and $u$, respectively. 
\eqref{eq:infimum dot phi ori} can be rewritten as 
\begin{align}
\label{eq:infimum dot phi analytical}
    \min_{u \in \mathcal{U}} \dot{\phi}(x, u) &=  {L_f\phi} + \sum_{i = 1}^{n_u} {L_g\phi}^{[i]} \mathbb{I}_{{L_g\phi}^{[i]}}\{u^{[i]}_{min}, u^{[i]}_{max}\}, \\ \nonumber 
\end{align}
where we define operator $\mathbb{I}$ as $\mathbb{I}_{A}\{B,C\} = B$ if $A \geq 0$, and $\mathbb{I}_{A}\{B,C\} = C$ if $A < 0$. Therefore, Problem \ref{def:problem 1} becomes 
\begin{problem}[\textit{Local Manifold Positiveness}]
\label{eq: local manifold problem}
Find $\Theta = \{k_1, k_2, \cdots, k_n\}$, such that
\begin{align}
    & \forall_{x \in \mathcal{X}, \phi_\Theta = 0}, {L_f\phi} + \sum_{i = 1}^{n_u} {L_g\phi}^{[i]} \mathbb{I}_{{L_g\phi}^{[i]}}\{u^{[i]}_{min}, u^{[i]}_{max}\} < 0,
\end{align}
where \textbf{local} denotes the state space $\mathcal{X}$ which is a subset of $\mathbb{R}^{n_x}$, and \textbf{manifold} denotes we are considering the states on the manifold defined by $\phi_\Theta = 0$.
\end{problem}

\begin{remark}
The state-of-the-art SOSP-based safety index synthesis methods try to synthesis $\phi_\Theta$, such that $\forall_{x\in\mathbb{R}^{n_x}},\exists_{u\in\mathcal{U}}\dot\phi_{\Theta}(x,u) < 0$, which is a \textit{global positiveness} constraint. Our formulation only requires the existence of safe control on a local set of critical states, i.e. \textit{local manifold positiveness} constraint. Therefore, our method considers a larger solution space; hence it is easier to find a solution, and the synthesized safety index is less conservative.  
\end{remark}


\subsection{The Refute Problem}\label{subsec:refute_prob}
We can show that the \textit{local manifold positiveness constraint} in Problem \ref{eq: local manifold problem} is satisfied by showing its refute set is empty. The refute set is constructed as:
\begin{equation}
    \begin{split}
    \begin{cases}
      {L_f\phi} + \sum_{i = 1}^{n_u} {L_g\phi}^{[i]} \mathbb{I}_{{L_g\phi}^{[i]}}\{u^{[i]}_{min}, u^{[i]}_{max}\} \geq 0\\
      x \in \mathcal{X} \\ 
      \phi_\Theta = 0
    \end{cases}
        \label{eq:refute first}
    \end{split}
\end{equation}

\rebuttal{Denote} $\mathcal{N},\mathcal{M}\subseteq \{1,2,\ldots,n_u\}$, $\mathcal{N} + \mathcal{M} = \{1,2,\ldots,n_u\}$ and $\mathcal{N}\cap\mathcal{M} = \emptyset$. We also denote $\forall i \in \mathcal{N}, {L_g\phi}^{[i]} \geq 0$, and $\forall i \in \mathcal{M}, {L_g\phi}^{[i]} \leq 0$. Then, \eqref{eq:refute first} corresponds to $2^{n_u}$ instances of refute set,
since each ${L_g\phi}^{[i]}$ can be either negative or nonnegative, yielding $2^{n_u}$ combinations of $\{\mathcal{N},\mathcal{M}\}$. Then each instance of \eqref{eq:refute first} can be rewritten as:
 \begin{equation}
    \begin{split}
    \begin{cases}
      \gamma^*_0 := {L_f\phi} + \sum_{i \in \mathcal{N}} {L_g\phi}^{[i]}u^{[i]}_{min} + \sum_{i \in \mathcal{M}} {L_g\phi}^{[i]}u^{[i]}_{max} \geq 0\\
      \gamma^*_i := S(x)_i \geq 0, i = 1,2,\cdots,n_S \\
      \gamma^*_{n_S+i} := {L_g\phi}^{[i]} \geq 0, i \in \mathcal{N} \\ 
      \gamma^*_{n_S+|\mathcal{N}|+i} := -{L_g\phi}^{[i]} \geq 0, i \in \mathcal{M} \\
      \zeta^* := \phi_\Theta = 0
    \end{cases}
        \label{eq:refute second}
    \end{split}
\end{equation}
\Cref{theo:posi} enables us to turn the emptiness problem for \eqref{eq:refute second} to a feasibility problem similar to \eqref{eq:pos}. As a result, we can show \eqref{eq:refute second} is empty using the following condition:
\begin{align}
\label{eq: refute 3rd}
    &\exists q_1 \in \mathbb{R}[x], \exists p_i\in\text{SOS},\forall i, \\ \nonumber
\text{ s.t. } & \gamma = p_0 + p_1\gamma_0^* + \ldots +  p_s\gamma_N^* \\ \nonumber
    & + p_{01}\gamma_0^*\gamma_1^* + \ldots + p_{012\ldots N}\gamma_0^*\ldots\gamma_N^*, \\ \nonumber 
     & \zeta = q_1 \zeta^*, \\ \nonumber
     & \gamma+\zeta+1 = 0 ,
\end{align}
where $N = n_S+|\mathcal{N}|+|\mathcal{M}|$.
Therefore, the Problem \ref{eq: local manifold problem} can be turned into the following equivalent problem:

\begin{problem}[\textit{Refute}]
\label{eq: refute problem}
find $\Theta = \{k_1, k_2, \cdots, k_n\}$, such that
\begin{align}
    \forall \mathcal{N,M,} \text{ \eqref{eq: refute 3rd} holds}
\end{align}
\end{problem}



\begin{remark}
It is noteworthy that \Cref{theo:posi} requires $\gamma^*$ and $\eta^*$ to be polynomials, which limits the applicability of current SOSP-based safety index synthesis methods.
On the other hand, as long as the system dynamics satisfy \Cref{def: subs}, our proposed method can be applied.
\end{remark}

\rebuttal{
\begin{remark}
Note that since ${L_g\phi}$ depends on $x$ and state space $\mathcal{X}$ is bounded, not all combinations of $\{\mathcal{N}, \mathcal{M}\}$ are possible. For those impossible combinations of $\{\mathcal{N}, \mathcal{M}\}$, \eqref{eq:refute second} is directly empty, which further eliminates the number of \eqref{eq:refute second} to be considered. 
\end{remark}
}

\subsection{The Nonlinear Programming Problem}\label{SDP}
To efficiently search for the existence of SOS polynomials $\{p_i\}$ and polynomials \rebuttal{$q_1$}, \rebuttal{we set $p_i$ to be a positive scalar $\alpha_i$ for all $i\geq1$ and $q_1$ to be a scalar $\beta_1$}.
Hence, a simplified condition of \eqref{eq: refute 3rd} can be defined:
\begin{align}
\label{eq: probsufficient}
    &\exists \rebuttal{\beta_1} \in \mathbb{R}, \exists \alpha_i \geq 0, \forall i\geq 1, \\ \nonumber
\text{ s.t. } & p_0 = -\alpha_1\gamma_0^* \ldots - \alpha_s\gamma_N^* - \alpha_{01}\gamma_0^*\gamma_1^* -\ldots\\ \nonumber
        &\quad \quad - \alpha_{012\ldots N}\gamma_0^*\ldots\gamma_N^* - \beta_1\zeta^* - 1  \in\text{SOS},
\end{align}



By searching for limited types of SOS polynomials, the existence of $\rebuttal{\beta_1}, \alpha_i$ satisfying \eqref{eq: probsufficient} is sufficient to satisfy the constraints in \eqref{eq: refute 3rd}.

Denote $\Omega = [\beta_1, \alpha_1, \alpha_2, \ldots , \alpha_{012\ldots N}]$ as the decision vector for \eqref{eq: probsufficient}. To solve \eqref{eq: probsufficient}, suppose the degree of $p_0$ is $2d$, and we first do a sum-of-squares decomposition of $p_0$ such that $p_0 = Y^\top Q^*(\Theta, \Omega) Y$, where $Q^*$ is symmetric and $Y = \begin{bmatrix}1 & x^{[1]} & x^{[2]} & \ldots & x^{[n]} & x^{[1]}x^{[2]} & \ldots & \big(x^{[n]}\big)^d \end{bmatrix}$ and $x^{[i]}$ is the entry in $x$. Specifically, for off-diagonal terms $Q^*_{ij}$ as the element of $Q^*$ at $i$-th row and $j$-th column ($i \neq j$), \revision{assuming that the} coefficient of the term $Y_iY_j$ in $p_0$ is $w_{ij}$, we set $Q^*_{ij} = \frac{w_{ij}}{2}$. 
With decomposed $Q^*$, the condition of \eqref{eq: probsufficient} can be rewritten as:
\begin{equation}
    \begin{split}
        & \det(Q^*(\Theta, \Omega)_k) > 0, \forall k,\\
    \end{split}
    \label{eq: sdp}
\end{equation}
where $Q^*(\Theta, \Omega)_k$ denotes the $k \times k$ submatrix consisting of the first $k$ rows and columns of $Q^*(\Theta, \Omega)$. 

According to Problem \ref{eq: refute problem}, there are $2^{n_u}$ sets of condition \eqref{eq: refute 3rd} and hence $2^{n_u}$ sets of condition \eqref{eq: probsufficient}. Denote the $\Omega_j = [\beta_1^j, \alpha_1^j, \alpha_2^j, \ldots , \alpha_{012\ldots N}^j]$ as the decision vector for $j$-th set of condition \eqref{eq: probsufficient}. With \eqref{eq: sdp}, Problem \ref{eq: refute problem} can be rewritten as:

\begin{problem}[\textit{Nonlinear Programming}]
\label{probfinal}
Find $[\Theta, \Omega_1, \Omega_2, \cdots, \Omega_{2^{n_u}}]$, such that
\begin{align}
    & \det(Q^*(\Theta, \Omega_j)_k) > 0, \forall k, \forall j, \\ \nonumber
    & \alpha_i^j > 0, \rebuttal{\beta_1^j}\in\mathbb{R}, \forall i, \forall j.
\end{align}
\end{problem}

Note that Problem \ref{probfinal} is a nonlinear programming problem without any objective. Hence, any arbitrary objective can be added to Problem \ref{probfinal}, and \rebuttal{it can be solved 
by off-the-shelf nonlinear programming solvers}.

\begin{remark}
Note that the state-of-the-art SOSP-based safety index synthesis methods 
need to solve multiple sets of SOSP iteratively.
\rebuttal{
On the other hand, Problem \ref{probfinal} demonstrates that our method just needs to solve one nonlinear program problem.
However, since Problem \ref{probfinal} involves enumerating at most $2^{n_u}$ combinations of constraint sets, which might not scale to high-dimensional systems (the scalability depends on the bound of $\mathcal{X}$).}
\end{remark}

\section{Properties of General Safety Index Design}\label{sec:proof}
This section proves that we can obtain a feasible safety index design by solving Problem \ref{probfinal}. The main result is summarized in the following theorem.

\begin{theorem}[Feasibility of General Safety Index Design]
The solution of Problem \ref{probfinal}   is also a solution of Problem \ref{def:problem 1}.
\label{theo1}
\end{theorem}

\begin{proof}
Denote the solution of Problem \ref{probfinal} as  $\Theta_3$, the solution of Problem \ref{eq: refute problem} as $\Theta_2$, the solution of Problem \ref{eq: local manifold problem} as $\Theta_1$,
and the solution of Problem \ref{def:problem 1} as $\Theta^o$.


\textbf{Relationship between $\Theta_1$ and $\Theta^o$}: According to \eqref{eq:infimum dot phi ori}, the following condition holds
\vspace{-5pt}
\begin{align}
\label{eq:inf dot phi is what}
    \min_{u \in \mathcal{U}} \dot{\phi}(x, u) &= {L_f\phi} +  \sum_{i = 1}^{n_u} \min_{u \in \mathcal{U}} {L_g\phi}^{[i]}u^{[i]},
    \vspace{-5pt}
\end{align}
where $\min_{u \in \mathcal{U}} {L_g\phi}^{[i]}u^{[i]} = {L_g\phi}^{[i]}\sup_{u\in\mathcal{U}}u^{[i]}$ if ${L_g\phi}^{[i]} < 0$, and $\min_{u \in \mathcal{U}} {L_g\phi}^{[i]}u^{[i]} = {L_g\phi}^{[i]}\inf_{u\in\mathcal{U}}u^{[i]}$ if ${L_g\phi}^{[i]} \geq 0$. Hence, the condition $\min_{u \in \mathcal{U}} \dot{\phi}(x, u)$ is equivalent to the condition ${L_f\phi} + \sum_{i = 1}^{n_u} {L_g\phi}^{[i]} \mathbb{I}_{{L_g\phi}^{[i]}}\{u^{[i]}_{min}, u^{[i]}_{max}\}$. So $\Theta_1$ is an instance of $\Theta_o$.

\textbf{Relationship between $\Theta_1$ and $\Theta_2$}:
  For the constraint $\forall_{x \in \mathcal{X}, \phi_\Theta = 0}, {L_f\phi} + \sum_{i = 1}^{n_u} {L_g\phi}^{[i]} \mathbb{I}_{{L_g\phi}^{[i]}}\{u^{[i]}_{min}, u^{[i]}_{max}\} < 0$, its refute certification is $\forall_{x \in \mathcal{X}, \phi_\Theta = 0}, {L_f\phi} + \sum_{i = 1}^{n_u} {L_g\phi}^{[i]} \mathbb{I}_{{L_g\phi}^{[i]}}\{u^{[i]}_{min}, u^{[i]}_{max}\} \geq 0$.
  By introducing auxiliary variables, the refute certifications of the constraints in Problem \ref{eq: local manifold problem} can be written as  \eqref{eq:refute second}. By \revision{\Cref{theo:posi}}, \cref{prop} for $\Gamma$ and definition for $\Phi$~\cite{parrilo2003semidefinite}, we know  \eqref{eq:refute second} is equivalent to \eqref{eq: refute 3rd}. Hence, Problem \ref{eq: local manifold problem} is equivalent to Problem \ref{eq: refute problem}, which indicates $\Theta_2$ is an instance of $\Theta_1$.
  

\textbf{Relationship between $\Theta_2$ and $\Theta_3$}:
The equivalent condition to \eqref{eq: probsufficient} is that $Q^*$ is positive semidefinite~\cite{parrilo2003semidefinite}. The condition \eqref{eq: sdp} proves that $Q^*$ is positive definite according to Sylvester's criterion.
Hence, the decision vectors $\{\Omega_1, \Omega_2, \cdots, \Omega_{2^{n_u}}\}$ satisfying condition \eqref{eq: sdp} also satisfy condition \eqref{eq: probsufficient}.


The simplest SOS is a positive constant scalar, \ie $\alpha \geq 0 \in \text{SOS}$. Similarly, the simplest polynomial is a constant scalar, \ie $\beta \in \mathbb{R}$. Hence, by substituting $p_i$ with $\alpha_i\geq0$ for $i = 1,2,\ldots$, and substituting \rebuttal{$q_1$} with $\rebuttal{\beta_1} \in \mathbb{R}$, condition \eqref{eq: probsufficient} can be rewritten as \eqref{eq: refute 3rd}. Therefore, the existence of $\{\Omega_1, \Omega_2, \cdots, \Omega_{2^{n_u}}\}$ satisfying the conditions in Problem \ref{probfinal} implies the satisfaction of conditions in Problem \ref{eq: refute problem}. Therefore, $\Theta_3$ is an instance of $\Theta_2$.

In summary, $\Theta_3$ is an instance of $\Theta_2$, hence an instance of $\Theta_1$, and hence an instance of $\Theta^o$, which verifies the claim.
\end{proof}

\section{Numerical Study}\label{sec: exp}

\subsection{Robot Arm Numerical Study Setup}
We evaluate our method in 2D robot arm systems with different degrees of freedom (DOF). The experimental platform is illustrated in \Cref{fig:nDOFS robot arm}. The link length of the robot is 1 meter. The obstacle is set as a half plane 
$0.5 \times \mathfrak{n}$
meter away from the robot base with $\mathfrak{n}$DOF.

\begin{wrapfigure}{r}{0.15\textwidth}
    \centering
    \includegraphics[width=0.15\textwidth]{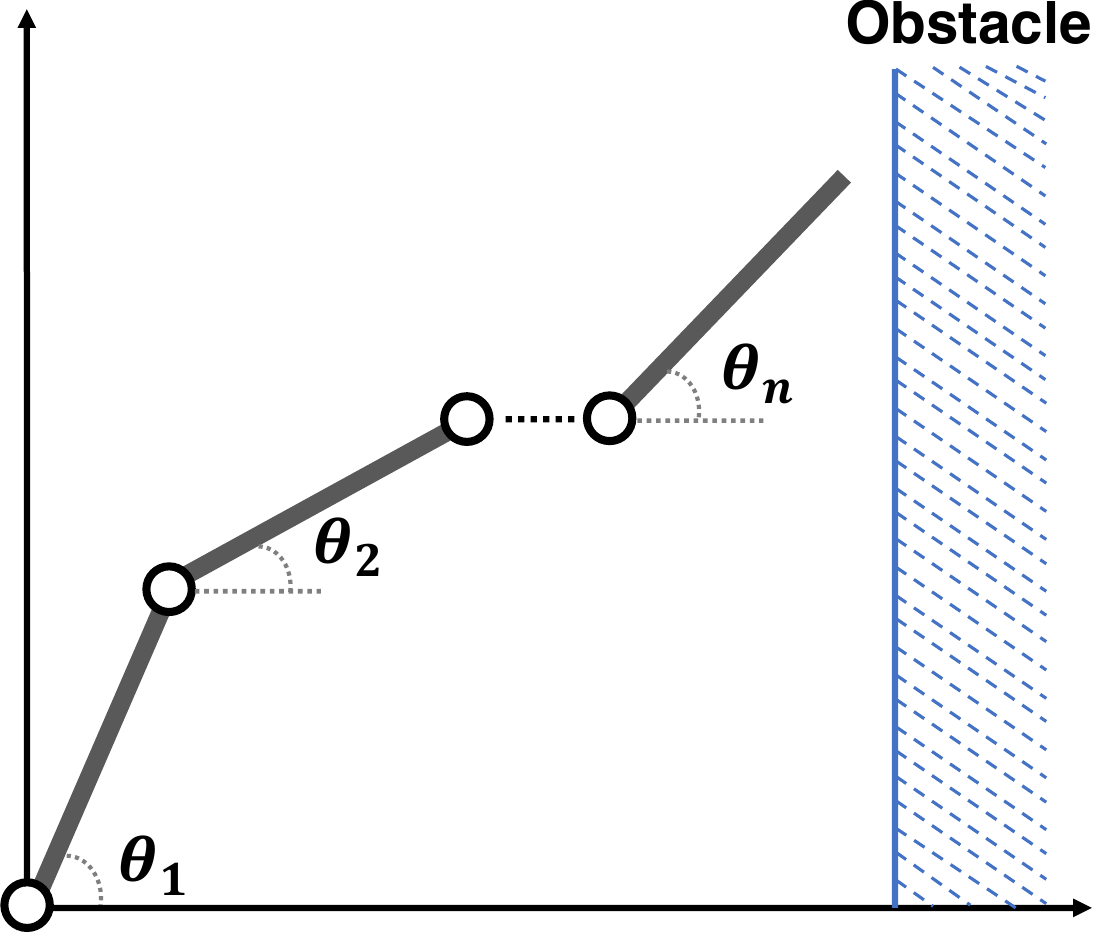}
    \caption{$\mathfrak{n}$DOFs robot manipulator.}
    \label{fig:nDOFS robot arm}
\vspace{-10pt}
\end{wrapfigure}

The state space includes 1) each joint angle and 2) each joint velocity. \rebuttal{Each joint angle is bounded by a sub-interval within $[\frac{\pi}{3}, \frac{2\pi}{3}]$ and each joint velocity is bounded within $[-1,1]$}. The control inputs are accelerations of each joint. 
The acceleration is limited within $[-1,1]$. 
The user-defined safety specification ($\phi_0$) requires the end effector to avoid collision with the wall. 



\subsection{Robot Arm Running Example}
In this subsection, we introduce a running example showing 1DOF robot arm safety index design via SOSP.
Consider 1) robot state $x = [\theta, \dot{\theta}]$, where $\theta \in [\pi/3,2\pi/3]$, $\dot\theta \in [-1,1]$;  2) robot control $u = [\ddot{\theta}]$, where $\ddot \theta \in [-1, 1]$.
\rebuttal{Consider the user-defined safety index as $\phi_0 =  \cos\theta - 0.5$}. 
Then, safety index becomes $\phi = \cos\theta -0.5 - k\sin\theta\dot \theta$,
with $\dot\phi = -\sin\theta \dot \theta - k \cos\theta \dot\theta^2 - k\sin\theta\ddot\theta$.
Since $\sin\theta > \frac{\sqrt{3}}{2}$, the minimum $\dot \phi$ is achieved when $\ddot \theta = 1$, i.e. $\min_{\ddot\theta} \dot \phi = -\sin\theta \dot \theta - k \cos\theta \dot\theta^2 - k\sin\theta$.
So the fundamental condition for nonempty set of safe control is: 
\vspace{-3pt}
\begin{equation}
\label{fundamental}
    \left\{
        \begin{array}{cl}
         & f_1  = -\sin\theta\dot\theta - k\cos\theta\dot\theta^2 - k\sin\theta < 0 \\ 
         & f_2 = \theta - \frac{\pi}{3} \geq 0\\
         & f_3 = \frac{2\pi}{3} - \theta \geq 0 \quad\quad\quad \\
         & f_4 = 1 - \dot\theta^2 \geq 0 \quad\quad\quad \\
         & h_1 = \cos\theta - 0.5 - k \sin\theta \dot \theta  = 0 \quad\quad\quad 
        \end{array}
    \right.
    \vspace{-3pt}
\end{equation}
Next, the equivalent refute set for \eqref{fundamental}  can be constructed by 
replacing $f_1 < 0$ in \eqref{fundamental} with $f_1 \geq 0$, which indicates that when $\phi$ hits zero, no $x \in \mathcal{X}$ would satisfy $\min_{\ddot \theta}\dot \phi > 0$. By projecting the non-polynomial terms of \eqref{fundamental} into higher dimensional polynomial manifold, i.e. substituting $\sin\theta = x_1, \cos\theta = x_2, \dot\theta = x_3$, 
 the refute set becomes:
\begin{equation}
\label{hiya}
    \left\{
        \begin{array}{cl}
         & f_1 = -x_1x_3 - kx_2x_3^2 - kx_1 \geq 0\\
         & f_2 = x_1 \geq \frac{\sqrt{3}}{2} \\
         & f_3 = 1 - x_1 \geq 0 \\
         & f_4 = 1 - x_3^2 \geq 0 \\
         & h_1 = x_2 - 0.5 - k x_1 x_3  = 0\\
         & h_2 = x_1^2 + x_2^2 - 1 = 0
        \end{array}
    \right.
    \vspace{-3pt}
\end{equation}
According to \Cref{theo:posi}, the emptiness of \eqref{hiya} is equivalent to the existence of $p_0, p_1, p_2, p_3, p_4 \in \text{SOS polynomials}$, $q_1, q_2 \in 
\text{ polynomials}$, such that: 
\vspace{-2pt}
\begin{align}\nonumber
    p_0 + p_1f_1 + p_2f_2 + p_3f_3 + p_4f_4 + q_1h_1 + q_2h_2 + 1 = 0
    \vspace{-2pt}
\end{align}

By considering a reduced parameter space, the emptiness of \eqref{hiya} in equivalent to the existence of $p_1, p_2, p_3, p_4 \in \mathbb{R}^+$, and $q_1, q_2 \in \mathbb{R}$, such that
\vspace{-2pt}
\begin{align}\nonumber
    p_0 = -(p_1f_1 + p_2f_2 + p_3f_3 + p_4f_4 + q_1h_1 + q_2h_2 + 1) 
    \vspace{-2pt}
\end{align}
is a sum of squares polynomial, where $[p_1, p_2, p_3, p_4,q_1, q_2]$ can be searched via nonlinear programming.

\rebuttal{To design safety index via SOSP for higher DOF robot arm, we need to construct $n$ sets of inequalities and equalities similar to \eqref{hiya} for the refute set of nDOF case, due to the fact that all joints are independent with each other.
}





\subsection{Robot Arm Results}
To solve for safety index design, we first provide reference $[\Theta, \Omega_1, \Omega_2, \cdots, \Omega_{2^{n_u}}]$ via random sampling, and then solve Problem \ref{probfinal} using MATLAB \verb+fmincon+ function \rebuttal{via interior point solver}. To evaluate safety index design, we randomly initialize the robot, and use safe set algorithm~\cite{liu2014control} to safeguard the robot for $2000$ time steps, where the robot takes the most dangerous reference control (i.e., increasing $\phi$).
We calculate the following statistics across the trials: (i) \textbf{Computing Time}: The average time to compute a safety index parameterization where all experiments are performed with a 2.3GHz Intel i5 Processor;  (ii) \textbf{Validness}: If safe control can be found over 1000 evaluation runs, we regard such parameterization of the safety index as valid; (iii) \textbf{Variance}: The variance of computed safety index parameterization across trials; (iv)  \rebuttal{\textbf{Feasibility}: The chances of solving a feasible solution for nonlinear programming across trials.}

\begin{table}[t]
\vspace{5pt}
\caption{\rebuttal{The performance of safety index design via SOS programming with 1000 random seeds (Robot Arms). Average computing time, validness, variance and the changes of solving a feasible solution are reported.}}
\vspace{-10pt}
\begin{center}
\begin{tabular}{c|c|c|c|c}
\toprule
 &  Time (s) & Validness (\%) & Variance & \rebuttal{Feasibility (\%)}\\
\hline
2DOF & 0.160 & 100 & 0.118 & \rebuttal{99.5}\\
4DOF & 0.180 & 100 & 0.042 & \rebuttal{95.6}\\
6DOF & 0.196 & 100 & 0.147 & \rebuttal{92.2}\\
10DOF & 0.303 & 100 & 0.247 & \rebuttal{91.6}\\
14DOF & 0.422 & 100 & 0.342 & \rebuttal{34.7}\\
\bottomrule
\end{tabular}
\end{center}
\vspace{-25pt}
\label{tab:sos_table}
\end{table}

\begin{figure}[htbp]
    \centering
    \vspace{-5pt}
    \includegraphics[width=0.55\columnwidth]{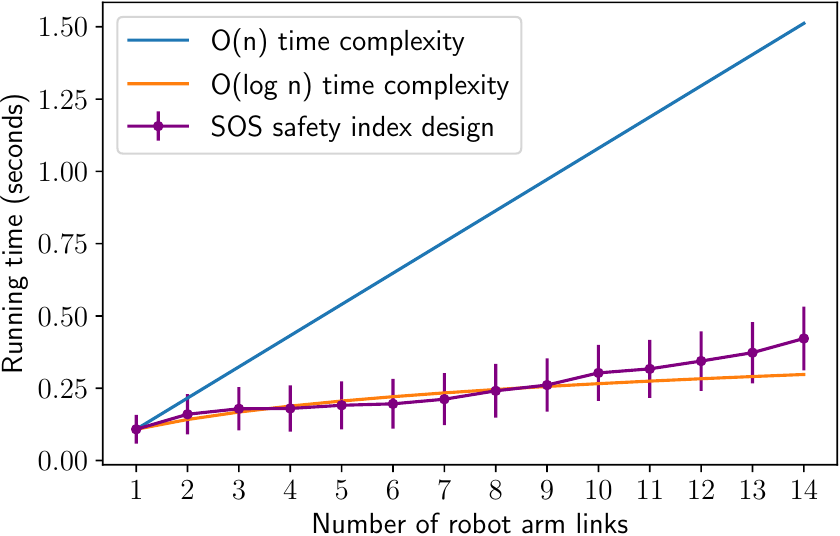}
    \caption{SOS running time in comparison with time complexities of $\mathcal{O}(n)$ and $\mathcal{O}(\log n)$. The purple vertical lines represent standard deviation of running time over 5000 random seeds.}
    \label{fig:SOS time}
    \vspace{-10pt}
\end{figure}
The results are summarized in \Cref{tab:sos_table}. 
SOS programming achieves 100\% validness (i.e., nonempty set of safe control for all states) across all degrees of freedom. \rebuttal{The changes of solving a feasible solution for nonlinear programming decreases as the degree of freedom increases.}
Moreover,
 the computing time scales logarithmically as the degree of freedom increases as shown in \Cref{fig:SOS time}, indicating the efficiency of 
our proposed method
for safety index design in complex systems.

\subsection{Autonomous Driving Results}
\rebuttal{
Additionally, we apply our method to design safety index for autonomous vehicles. The vehicle is a 4-state unicycle model, where the velocity and heading angle are bounded within $[0,1]$ and $[0, \frac{\pi}{2}]$, respectively. The control inputs are acceleration and angular velocity, where acceleration is bounded within $[-1,1]$ and angular velocity is bounded within $[-1,1]$. We consider one static circular obstacle, and the user-defined safety specification $(\phi_0)$ requires the vehicle to avoid collision with the obstacle.}


As shwon in \Cref{tab:sos_table_drive}, our proposed method can efficiently and robustly synthesize safety index for vehicle systems. 
As for this setting, a valid safety index needs $k_1 > 1$ according to ~\cite{zhao2021issa}. In our experiments with 1000 random seeds, the minimum synthesized $k_1$ is $1.0001$, which indicates good optimality of solutions of our proposed method.

\begin{table}[t]
\vspace{5pt}
\caption{The performance of safety index design via SOS programming with 1000 random seeds (Autonomous Driving).}
\vspace{-10pt}
\begin{center}
\begin{tabular}{c|c|c|c}
\toprule
 &  Time (s) & Validness (\%) & Variance \\
\hline
4-State Unicycle & 0.066 & 100 & 0.011 \\
\bottomrule
\end{tabular}
\end{center}
\vspace{-20pt}
\label{tab:sos_table_drive}
\end{table}

\section{CONCLUSIONS}
This paper proposed a framework for synthesizing the safety index for general control systems under control limits using sum-of-squares programming. Our approach leverages Positivstellensatz theorem to ensure the non-emptiness of safe control. The experimental results show that the synthesized safety index guarantees safety and our method is effective even in high-dimensional robot systems. \rebuttal{The proposed method is limited to white-box analytical dynamics. Future research could focus on safety index synthesize for black-box dynamics.}

\bibliographystyle{IEEEtran.bst}
\bibliography{reference}

\addtolength{\textheight}{-12cm}   











\end{document}